%% file: arxiv.tex
\documentclass{article}

\PassOptionsToPackage{numbers, compress}{natbib}

\usepackage[final]{styles/neurips_mlsb_2024}

\usepackage{booktabs}            
\usepackage{multirow}            
\usepackage{amssymb,amsfonts}            
\usepackage{graphicx}            
\usepackage{duckuments}          
\usepackage{xcolor, colortbl}

\usepackage{amsthm}
\usepackage{thmtools,thm-restate}
\usepackage{enumitem}
\usepackage{todonotes}
\usepackage{tikz}
\usepackage{tikz-cd}
\usepackage{caption}
\usepackage{subcaption}
\usepackage{mathrsfs}
\usepackage{tabularx}

\usepackage[utf8]{inputenc} 
\usepackage[T1]{fontenc}    
\usepackage{hyperref}       
\usepackage{url}            
\usepackage{nicefrac}       
\usepackage{microtype}      
\usepackage{adjustbox}
\usepackage{etoolbox}       
\input{styles/math_commands.tex}
\newcommand{\gray}[1]{\textcolor{gray}{#1}}

\newtheorem{definition}{Definition}[section]
\newtheorem{proposition}{Proposition}[section]
\newtheorem{theorem}{Theorem}[section]
\newtheorem{lemma}{Lemma}[section]

\newtoggle{advanced}

\togglefalse{advanced}

\newcommand{\pmn}[1]{{\scriptsize $\pm$ {#1}}}

\title{Multi-Scale Protein Structure Modelling with Geometric Graph U-Nets}

\author{%
  Chang~Liu$^{*}$ \\
  University of Cambridge \\
  \texttt{cl962@cantab.ac.uk}
  \And
  Vivian~Li$^{*}$ \\
  University of Cambridge \\
  \texttt{vdl21@cam.ac.uk}
  \And
  Linus~Leong \\
  University of Cambridge \\
  \texttt{yhll2@cantab.ac.uk} \\
  \And
  Vladimir~Radenkovic \\
  University of Cambridge \\
  \texttt{vr375@cam.ac.uk} \\
  \And
  Pietro~Li\`o \\
  University of Cambridge \\
  \texttt{pl219@cam.ac.uk}
  \And
  Chaitanya~K.~Joshi \\
  University of Cambridge \\
  \texttt{ckj24@cam.ac.uk} \\
}

\begin{document}

\maketitle
\def\thefootnote{*}\footnotetext{Equal contributions. See Appendix \ref{app:contrib} for detailed author contributions.}

\begin{abstract}
Geometric Graph Neural Networks (GNNs) and Transformers have become state-of-the-art for learning from 3D protein structures. However, their reliance on message passing prevents them from capturing the hierarchical interactions that govern protein function, such as global domains and long-range allosteric regulation. 
In this work, we argue that the network architecture itself should mirror this biological hierarchy. 
We introduce Geometric Graph U-Nets, a new class of models that learn multi-scale representations by recursively coarsening and refining the protein graph. 
We prove that this hierarchical design can theoretically more expressive than standard Geometric GNNs. 
Empirically, on the task of protein fold classification, Geometric U-Nets substantially outperform invariant and equivariant baselines, demonstrating their ability to learn the global structural patterns that define protein folds. 
Our work provides a principled foundation for designing geometric deep learning architectures that can learn the multi-scale structure of biomolecules.
\end{abstract}

\section{Introduction}

Proteins are not just bags of atoms; they are hierarchically organized structures, where local motifs like alpha-helices and beta-sheets assemble into functional domains, which in turn form the global quaternary structure \citep{richardson2008}. 
This multi-scale organization is critical for protein function, as local interactions (e.g., hydrogen bonds) combine with long-range effects (e.g., allosteric regulation) to determine how a protein behaves in its biological context \citep{motlagh2014}.
Standard geometric GNNs generally rely on local, fixed-radius message passing \citep{gilmer2017neural} and thus overlook the hierarchical, compositional organisation of proteins.
Existing fixes such as virtual nodes \citep{sestak2024vn} or ad hoc long-range edges \citep{ingraham2023illuminating, jumper2021alphafold} are heuristic and fail to provide a principled, continuous multi-scale representation of protein structure. 

Here, we introduce the Geometric Graph U-Net, a novel architectural blueprint that explicitly models the multi-scale nature of proteins by integrating geometric pooling layers within a hierarchical U-Net structure \citep{ronneberger2015u}. 
Further, we provide a rigorous theoretical foundation for this approach, extending the Geometric Weisfeiler-Leman (GWL) test \citep{joshi2023expressive} to prove that our hierarchical architecture is strictly as or more expressive than standard message-passing GNNs. 
Finally, we empirically evaluate our model's performance on protein fold classification, where Geometric U-Nets consistently improve performance over invariant and equivariant baselines.

\begin{figure}[t!]
    \centering
    \includegraphics[width=\linewidth]{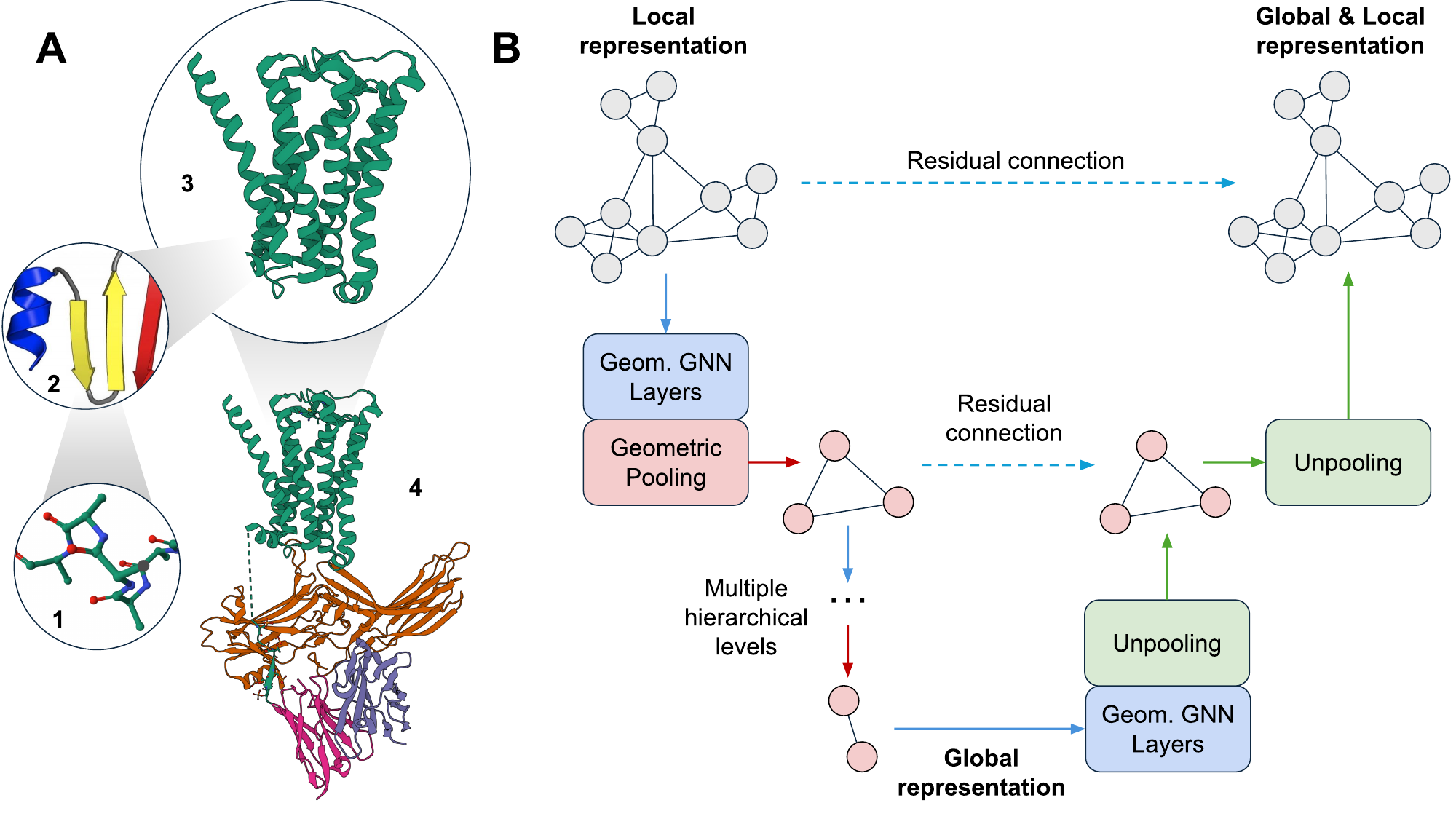}
    \caption{\textbf{Geometric Graph U-Nets for multi-scale protein representation learning.} (A) Protein function is governed by a structural hierarchy. Local secondary structures fold into tertiary domains, which assemble into quaternary complexes to perform complex functions like allosteric signaling. Standard GNNs struggle to capture these long-range interactions. (B) The Geometric Graph U-Net is a biologically inspired architecture that encodes a protein graph by progressive coarsening, moving from local representations to global ones. The decoder then refines this information, using residual connections to re-inject high-resolution local information at each level. Thus, the final representation integrates both global context and local geometric details.}
    \label{fig:unet-diagram}
\end{figure}


\section{Background}

\textbf{Geometric Graph Neural Networks (GNNs)} have become the standard toolkit for representation learning of 3D atomic systems such as biomolecules \cite{zhang2025artificial}.
We define a geometric graph $\mathcal{G} = (\mathcal{A}, \mathbf{S}, \vec{\mathbf{V}}, \vec{\mathcal{X}})$ as an attributed graph with adjacency matrix $\mathbf{A}$ of size $N \times N$, scalar features  $\mathbf{S} \in \mathbb{R}^{N \times f}$ and geometric attributes: node coordinates $\vec{\mathcal{X}}\in \mathbb{R}^{N \times d}$ and optional vector features $\vec{\mathbf{V}}\in \mathbb{R}^{N \times d}$. Feature dimensions $f$ and $d$ represent the scalar feature channel dimensions and the Euclidean space dimensions, respectively. We say two geometric graphs are \textit{geometrically isomorphic} if the underlying attributed graphs are isomorphic and their geometric attributes are equivalent, up to global group actions like rotation and reflection.

\citet{duval2024hitchhikers} offers a comprehensive review of geometric GNN models and the different symmetries that are encodable in different message-passing layers. 
We consider two classes of geometric GNN architectures: using invariant \cite{schutt2017schnet, zhang2023protein, gasteiger2020fast, fan2022continuous} or equivariant \cite{jing2020learning, morehead2024geometry, satorras2021n} geometric GNN layers.
\textbf{Invariant GNN layers} only propagate invariant scalar features computed using geometric information within local neighbourhoods, following the equation:
    $$\vs_i^{(t+1)} \defeq \textsc{Upd} \Big( \vs_i^{(t)} \ ,
    \textsc{Agg} \left( \ldblbrace (\vs_i^{(t)}, \vec{\vv}_i^{(t)}), (\vs_j^{(t)},  \vec{\vv}_j^{(t)}), \vec{\vx}_{ij} \mid j \in \mathcal{N}_i \rdblbrace \right) \Big).$$
\textbf{Equivariant GNN layers} update both scalar and geometric vector features from iteration $t$ to $t+1$ via learnable aggregate and update functions, $\textsc{Agg}$ and $\textsc{Upd}$, respectively:
    $$\vs_i^{(t+1)}, \vec{\vv}_i^{(t+1)} \defeq \textsc{Upd} \Big( (\vs_i^{(t)}, \vec{\vv}_i^{(t)}) \ ,
    \textsc{Agg} \left( \ldblbrace (\vs_i^{(t)}, \vec{\vv}_i^{(t)}), (\vs_j^{(t)},  \vec{\vv}_j^{(t)}), \vec{\vx}_{ij} \mid j \in \mathcal{N}_i \rdblbrace \right) \Big),$$
where $\vec{\vx}_{ij} = \vec{\vx}_{i} - \vec{\vx}_{j}$ denote relative position vectors.
As illustrated above, each geometric GNN layer aggregates information from a node's neighbourhood, meaning information in GNNs spreads in terms of $k$-hop \textit{receptive fields}, where $k$ is the number of GNN layers.

Geometric GNNs excel in capturing local relationships between nodes, but they often struggle to model global relations, usually adding more layers to enlarge the receptive field.
This enlarged receptive field may cause \textit{over-squashing}, as fixed-size feature vectors cannot encode the added information \cite{alon2021bottleneck}, limiting model expressivity \cite{digiovanni2023oversquashing}.
In this paper, we will solve this problem through hierarchical GNNs which model multi-scale interactions.

\textbf{Hierarchical Geometric GNNs} enhance global representation by reducing node distances and enlarging the receptive field. Prior approaches fall into three main categories: 
(1) Local pooling clusters nodes at each GNN layer to form new nodes for the next layer, with approaches \cite{qi2017pointnet, qi2017pointnet++, qian2022pointnext, ying2019diffpooling, ranjan2020asap, liu2021hierpooling} ranging from basic clustering to learnable and attention-based adaptive pooling. Our research extends the benefits of pooling via a U-Net, of which the effectiveness has been demonstrated in prior studies \cite{hu2019runet,jiang2024vision,shen2023graph}. 
(2) Virtual nodes create artificial connections between distant nodes and can embed domain knowledge in edge construction, with methods \cite{zang2023molecular, gu2023heal, sestak2023vnegnn} decomposing geometric graphs into virtual supernodes for structural motifs. In contrast, we explore multi-layer hierarchies to model both intrinsic structures and global interactions. 
(3) Virtual edges enable information flows across distant residues, with some approaches \cite{ingraham2023illuminating} adding random long-range connections. 
Transformer networks \citep{vaswani2017attention, jumper2021alphafold} also fall into this category, as they can be viewed as message passing on fully-connected graphs \citep{joshi2020transformers}.
In contrast, our method builds edges by sampling nearby nodes, reducing randomness and forming long-range connections between structurally similar nodes via hierarchies.

\textbf{GNN Expressivity} formally characterizes how architectural ideas influence the class of functions a GNN can represent \citep{raghu2017expressive, xu2018how, morris2019weisfeiler}.
\citet{joshi2023expressive} proposes the Geometric Weisfeiler-Leman (GWL) test, a geometric adaptation of the Weisfeiler-Leman test that measures expressivity of Geometric GNNs by their ability to distinguish geometric graphs while preserving 3D symmetries.

The GWL test operates in iterations: in each iteration, graph nodes update their colour by aggregating both scalar and geometric information in its neighbourhood with $\mathfrak{G}$-orbit injective and $\mathfrak{G}$-invariant function over Lie groups $\mathfrak{G} = SO(d)$ or $\mathfrak{G} = O(d)$. Auxiliary geometric information variables are additionally updated by aggregating local geometric information around each node with injective and $\mathfrak{G}$-equivariant aggregation. The test then terminates when the colours of the nodes do not change from the previous iteration, and the two graphs are geometrically non-isomorphic if multi-sets of their node features are different: 
\begin{definition}
    Graphs $\mathcal{G}_1$ and $\mathcal{G}_2$ are \textbf{k-GWL distinguishable} ($\mathcal{G}_1\neq_{k\text{-}GWL}\mathcal{G}_2$) if there exists a test iteration $i\leq k$ for which their multisets of node features are different.
\end{definition}

\citet{joshi2023expressive} additionally defines the Invariant GWL (IGWL) test that only updates node colours using $k$-hop neighbourhood scalar and geometric information with a $\mathfrak{G}$-orbit injective and invariant function without updating the geometric representation of nodes:
\begin{definition}
    Graphs $\mathcal{G}_1$ and $\mathcal{G}_2$ are \textbf{IGWL distinguishable} ($\mathcal{G}_1\neq_{IGWL}\mathcal{G}_2$) if there exists a test iteration $i$ for which their multisets of node features are different.
\end{definition}

Building on GWL and IGWL, \citet{joshi2023expressive} provides an upper bound on the expressive power of both invariant and equivariant GNNs by proving that equivariant GNNs are at most as powerful as GWL at distinguishing non-isomorphic geometric graphs, while invariant GNNs are limited by the IGWL at the same task. Thus, a theoretical framework for proving the expressivity of new GNN architectures can be established by showing the GWL- or IGWL-distinguishability of graphs under the new model.
We introduce an additional definition for analyzing pooling expressivity:
\begin{definition}
    Graphs $\mathcal{G}_1$ and $\mathcal{G}_2$ are \textbf{currently GWL or IGWL distinguishable} ($\mathcal{G}_1\neq_{C\text{-}GWL/IGWL}\mathcal{G}_2$) if their multisets of node features are different at the current iteration $i$, according to the respective GWL or IGWL test.
\end{definition}


\section{Geometric Graph U-Nets}

Our goal in designing a geometric pooling operation is threefold. The operation must be (1) geometrically meaningful, preserving the 3D structure and respecting physical symmetries; (2) computationally efficient, allowing for the construction of deep, multi-level hierarchies; and (3) architecturally compatible with both invariant and equivariant Geometric GNNs layers.

\subsection{Geometric Pooling Layers}

To define the geometric pooling layers, we follow the Select-Reduce-Connect (SRC) framework introduced in \citet{grattarola2021pooling}, where \texttt{POOL} is the combination of three functions: \textit{selection} (\texttt{SEL}), \textit{reduction} (\texttt{RED}), and \textit{connection} (\texttt{CON}). \texttt{SEL} clusters the nodes of the input graph into subsets called \textit{supernodes}, producing cluster assignment matrix $\mathbf{C}$ where $c_{ij}$ is the membership score of node $i$ to supernode $j$. \texttt{RED} creates the pooled node features by aggregating the features of the nodes assigned to the same supernode:
\[\mathtt{RED}(\mathcal{G}, \mathbf{C})\mapsto \mathbf{S}^{P}\]
where $|P|$ denotes node cardinality of the pooled graph $\mathcal{G}_P$  Finally, \texttt{CON} constructs new edges between supernodes:
\[\mathtt{SEL}(\mathcal{G}, \mathbf{C})\mapsto \mathbf{A}'\]
We outline the shared components for these functions among all of our proposed pooling layers below:

\textbf{Selection ($\texttt{SEL})$: } Inspired by PointNet++ \cite{qi2017pointnet++} and recognizing that each geometric graph node corresponds to physical coordinates, we employ the farthest point sampling (FPS) for graph coarsening, where sampled nodes  $\mathcal{V}_P = \{v_{1}, \ldots, v_{K}\}$ define the centres of supernode clusters $\mathbf{C} = \{ \mathbf{C}_1, \ldots, \mathbf{C}_K \}$, while the rest of the nodes are assigned to one or more supernodes based on proximity. By having supernodes adopt the coordinates of their central nodes, we ensure that the physical interpretability and properties of the coarsened graph are preserved. This method effectively sparsifies the graph by selectively thinning dense protein representations. We empirically set the sampling ratio for FPS in each pooling layer at 0.6.

\textbf{Reduction ($\texttt{RED})$: } The reduction function (\texttt{RED}) creates the pooled node features by aggregating the features of the vertices assigned to the same supernode. We extend reduction function $\texttt{RED}$ to update node coordinates $\vec{\mathcal{X}}\in\mathbb{R}^{N \times d}$ and vector features $\vec{\mathbf{V}}\in\mathbb{R}^{N \times v}$:
\[\mathtt{RED}(\mathcal{G,\mathbf{C}}) \mapsto (\mathbf{S}_P, \vec{\mathbf{V}}_P, \vec{\mathcal{X}}_P)\]

\textbf{Connection ($\texttt{CON})$: } The edge generation process in coarsened graphs is guided by the need for pooled graphs to retain regular structure based on the proximity of nodes. To achieve this, we adopt the K-nearest neighbours (K-NN) technique, with a fixed value of $k=16$ in our experiments, to form edges between supernodes, just as the original graphs in the benchmark datasets are constructed.

This design of $\texttt{SEL}$ and $\texttt{CON}$ allows for flexible, controlled, and adaptive graph size reduction and edge set design. This way, we aim to construct geometric pooling layers specifically designed for graphs structured from point clouds, as often used in biochemistry and material science applications.

\begin{figure}[ht]
    \centering
    \includegraphics[width=0.8\textwidth]{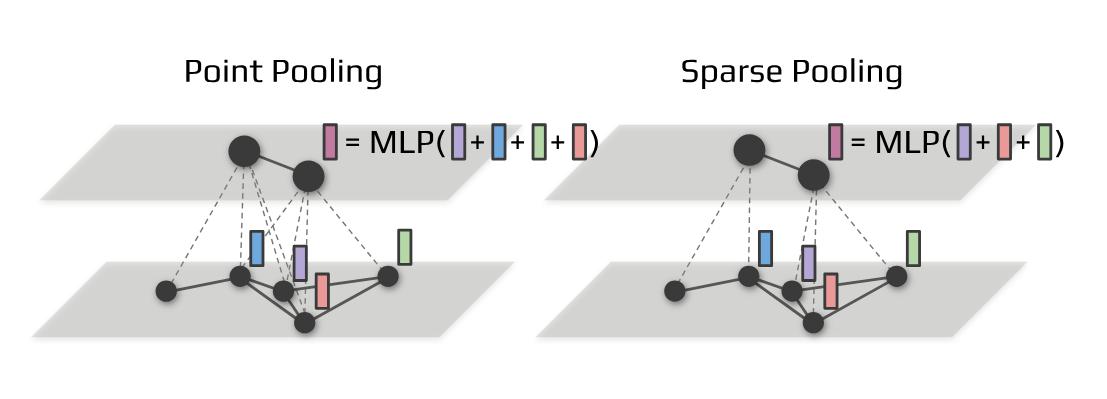}
    \caption{\textbf{Comparison of Point Pooling and Sparse Pooling layers.} Both pooling layers use farthest point sampling (FPS) to select supernode centers from the original graph. (Left) Point Pooling assigns all 1-hop neighbors of each supernode to its cluster, creating overlapping regions that capture local structural motifs. (Right) Sparse Pooling assigns each node to its nearest supernode, partitioning the graph into non-overlapping regions. Both layers aggregate features from assigned nodes to compute supernode representations.}
    \label{fig:pooling-comparison}
\end{figure}

\subsection{Point Pooling Layer}

Point Pooling Layers adapt the pooling operation used for point cloud data in PointNet++ \cite{qi2017pointnet++} to protein structures. This layer creates supernodes that represent local structural motifs by aggregating information from their immediate 1-hop neighborhood, akin to identifying a patch of the protein surface. For each pair of nodes $v_i\not\in \mathcal{V}_P$ and supernodes $v^{sn}_j$, where $v_i$ denotes the i-th node, $\mathcal{V}_P$ stands for the sampled node set, and $v^{sn}_j$ is the j-th supernode, we define the cluster assignment matrix $\mathbf{C}$, as the following:
\begin{equation}
    \mathbf{C}_{j, i} = 
    \begin{cases} 
        1 & \text{if } i \in \mathcal{N}_j \\
        0 & \text{otherwise}
    \end{cases}
\end{equation}
As a reduction function for scalar features, we use the message-passing layer from GIN \cite{xu2018how} to update the features of each pooled supernode $j \in \mathbf{C}$:
\begin{align}
    \mathbf{s}^{sn}_j \ = \ \text{MLP}(\mathbf{C}_{j}^\intercal\mathbf{S}) \ = \ \text{MLP}\bigl(s_j + \sum_{i\in \mathcal{N}_j}s_i\bigr),
\end{align}
where $\mathbf{C}_j$ represents row $j$ in the cluster assignment matrix. 
For the reduction of vector features, we find that a summation aggregation of neighbouring nodes' vector features of supernodes, without introducing non-linear transformations, provides the most effective results. For each pooled supernode $j \in \mathbf{C}$ we define reduction as:
\begin{align}
    \vec{\mathbf{v}}^{sn}_j \ = \ \mathbf{C}_j^\intercal\vec{\mathbf{V}} \ = \ \vec{\mathbf{v}}_j + \sum_{i\in\mathcal{N}_j} \vec{\mathbf{v}}_i,
\end{align}
which ensures that updated scalar and vector features of coarsened point clouds remain invariant and equivariant to both SO(3) transformations.
This pooling layer efficiently extracts representations of subgraphs surrounding each sampled node within the original graphs. In this way, supernodes encapsulate not only information about individual nodes they physically represent, but also about the broader molecular regions around them. 
While this layer possesses an injective reduction function for scalar features, the combined selection and reduction function ($\texttt{SEL} \circ \texttt{RED}$) is not injective and thus does not fully satisfy the sufficient conditions required for maintaining expressivity. Nevertheless, the simplicity of this layer, along with our later proof with IGWL for its expressivity, makes it a compelling choice for an initial pooling layer.

\subsection{Sparse Pooling Layer}

Sparse Pooling Layers assign each node to its nearest supernode, effectively partitioning the structure into distinct, non-overlapping regions. For each pair of nodes $v_i\not\in \mathcal{V}_P$ and 
supernodes $v^{sn}_j \in \mathcal{V}_P$ we define cluster assignment matrix $\mathbf{C}$, as the following:
\begin{equation}
    \mathbf{C}_{j,i} =
\begin{cases}
  1 & \text{if } i = \underset{k \neq j}{\arg\min} \, \|\mathbf{x}_j - \mathbf{x}_k\|_2 \text{ or } i = j \\
  0 & \text{otherwise}
\end{cases}
\end{equation}
For the reduction function, we use the summation of assigned nodes to update the scalar and vector features of supernodes:
\begin{align}
    \mathbf{s}_j^{sn} = \mathbf{C}_j^{\intercal}\mathbf{S}, \quad\quad\quad \vec{\mathbf{v}}_j^{sn} &= \mathbf{C}_j^{\intercal}\vec{\mathbf{V}}
\end{align}
The cluster assignment matrix of this layer is sparse, with each row containing a single non-diagonal, non-zero entry (set as 1) indicating supernode affiliation. This pooling layer is called \textit{Sparse} due to the constant cardinality $O(1)$ of each supernode \cite{grattarola2021pooling}. 
As we will see later, if the preceding message passing layers for sparse pooling are sufficiently expressive and satisfy the first condition of Theorem \ref{thm:maintain_expressivity}, this pooling layer satisfies the second condition and maintains the same expressive power.

\subsection{U-Nets with Geometric Pooling Layers}
\label{sec:U-Net}

We build upon the framework of \citet{jamasb2024evaluating}, which uses a Geometric GNN encoder to generate embeddings for 3D molecular graphs. We extend this encoder with a hierarchical U-Net architecture comprising encoder and decoder paths with pooling and unpooling blocks.
Our encoder consists of six message-passing layers with a geometric pooling layer inserted after every two message-passing layers, creating three hierarchical levels.
Each pooling layer progressively coarsens the graph by clustering nodes into supernodes, enabling connections between longer-range structural elements.

Formally, given an input graph $\mathcal{G}^{i-1}$ with $|\mathcal{V}^{i-1}| = N$ nodes, the $i$-th pooling block outputs a pooled graph $\mathcal{G}^{i}$ with a reduced number of nodes $|\mathcal{V}^{i}| = K \leq N$, where $\mathcal{G}^{i} = \mathtt{POOL}(\mathcal{G}^{i-1})$ and node features $(\mathbf{S}^{i}, \vec{\mathbf{V}}^{i}) = \mathtt{GNN}(\mathcal{G}^{i})$. Here, $\texttt{POOL}$ denotes the geometric pooling layer and $\texttt{GNN}$ represents the message-passing layer that updates the scalar and vector features $\mathbf{S}^{i}$ and $\vec{\mathbf{V}}^{i}$ of the pooled graph $\mathcal{G}^{i}$.
Thus, the U-Net architecture is compatible with any invariant or equivariant Geometric GNN layers. 

Each unpooling block integrates pooled and unpooled features through an unpooling layer. Given input graph $\mathcal{G}^{j-1}$ with $|\mathcal{V}^{j-1}| = K$ nodes, the $j$-th unpooling block outputs an unpooled graph $\mathcal{G}^{j}$ with a larger number of nodes $|\mathcal{V}^{j}| = N \geq K$, where $\mathcal{G}^{j} = \mathtt{UNPOOL}(\mathcal{G}^{j-1})$.
As supernodes retain physical information, the unpooled graph is reconstructed by reintroducing removed nodes with learnable scalar and vector features initialised to zeros, with a skip connection concatenating cached features from the corresponding pooling blocks to the unpooling output.
An overview of this hierarchical U-Net architecture is shown in Figure \ref{fig:unet-diagram}.

\textbf{Intuition.} 
The encoder path progressively coarsens the graph, moving from a fine-grained atomic representation to a coarse-grained, domain-level view. This allows subsequent message-passing layers to operate on nodes that represent entire sub-structures, enabling efficient propagation of information across long sequence distances that would be intractable in the original graph.
The decoder path refines this coarse representation back to the atomic level. Crucially, the skip connections re-inject high-resolution local information from the encoder at each stage. This architectural feature is vital, as it allows the model to use global context (e.g., the relative orientation of two domains) to make precise, localized predictions (e.g., about the specific atoms at the domain interface).

\section{Expressive Power of Geometric Pooling}
\label{sec:pooling-layer-expressivity}

We study expressivity through the lens of distinguishing non-isomorphic geometric graphs, adopting the Geometric Weisfeiler–Leman (GWL) test \citep{joshi2023expressive}. 
We then extend the hierarchical expressivity framework introduced for (non-geometric) pooling by \citet{bianchi2023expressive} and \citet{lachi2023graph} to the geometric setting, enabling analysis of when pooling preserves or strengthens GWL (and IGWL) distinguishing capacity.
We first introduce key definitions for analyzing pooling expressivity:
\begin{definition}
    An invariant graph pooling layer $\mathtt{POOL} = (\mathtt{SEL}, \mathtt{RED}, \mathtt{CON})$ \textbf{maintains expressivity} if it maps any pair of currently IGWL-distinguishable graphs (CIGWL) to a pair of IGWL-distinguishable graphs, i.e., $\mathcal{G}_{1}\neq_{CIGWL}\mathcal{G}_2 \Rightarrow \mathtt{POOL}(\mathcal{G}_1)\neq_{IGWL} \mathtt{POOL}(\mathcal{G}_2)$.
\end{definition}

\begin{definition}
    An equivariant graph pooling layer $\mathtt{POOL} = (\mathtt{SEL}, \mathtt{RED}, \mathtt{CON})$ \textbf{maintains expressivity} if there exists a $k$ such that it maps any pair of currently GWL-distinguishable graphs (CGWL) to a pair of $k$-GWL-distinguishable graphs, i.e., $\mathcal{G}_{1}\neq_{CGWL}\mathcal{G}_2 \Rightarrow \mathtt{POOL}(\mathcal{G}_1)\neq_{k-GWL} \mathtt{POOL}(\mathcal{G}_2)$ for any $k$.
\end{definition}

Maintaining expressivity means that pooling does not lose discriminative information. If two protein structures are distinguishable before coarsening (e.g., different fold types), they remain distinguishable after pooling. This is a crucial property: we want our hierarchical architecture to preserve the GNN's ability to tell proteins apart, even as we progressively abstract from atoms to domains.

\begin{definition}
    An invariant graph pooling layer $\mathtt{POOL} = (\mathtt{SEL}, \mathtt{RED}, \mathtt{CON})$ \textbf{increases expressivity} if it maps any pair of IGWL-indistinguishable graphs to a pair of IGWL-distinguishable graphs, i.e., $\mathcal{G}_{1} =_{IGWL}\mathcal{G}_{2} \Rightarrow \mathtt{POOL}(\mathcal{G}_{1})\neq_{IGWL}\mathtt{POOL}(\mathcal{G}_{2})$.
\end{definition}

\begin{definition}
    An equivariant graph pooling layer $\mathtt{POOL} = (\mathtt{SEL}, \mathtt{RED}, \mathtt{CON})$ \textbf{increases expressivity} if it maps any pair of $k$-GWL-indistinguishable graphs to a pair of $k$-GWL-distinguishable graphs, i.e., $\mathcal{G}_{1} =_{k-GWL}\mathcal{G}_{2} \Rightarrow \mathtt{POOL}(\mathcal{G}_{1})\neq_{k-GWL}\mathtt{POOL}(\mathcal{G}_{2})$.
\end{definition}

Increasing expressivity is a stronger property: pooling can actually make previously indistinguishable structures become distinguishable. This happens because the coarsened view exposes global structural patterns that are invisible at the atomic level. For example, two proteins might have locally similar motifs but globally different domain arrangements.

Let $\mathcal{X}_{\mathcal{G}_{i}}^{k\text{-(I)GWL}} = \left\{\!\left\{(s_{j}^{i}, \vec{\mathbf{v}}_{j}^{i}):j\in \mathcal{V}\right\}\!\right\}$ be the multi-set of $k$-GWL-discriminative or IGWL-discriminative scalar and vector node features for a graph $\mathcal{G}_{i}$.
From the following proposition, as long as $ \texttt{RED}\circ \texttt{SEL}$ uniquely maps unpooled nodes to pooled nodes, $\texttt{POOL}$ will maintain expressivity:

\begin{proposition}\label{prop:maintain_expressivity}
    Let $\mathtt{POOL} = (\mathtt{SEL}, \mathtt{RED}, \mathtt{CON})$ such that $\mathtt{RED}\circ \mathtt{SEL}: (\mathcal{X}_{\mathcal{G}}^{k\text{-(I)GWL}}) \mapsto \mathcal{X}_{\mathcal{G}^P}^{k\text{-(I)GWL}}$
    is injective on $\mathcal{X}_{\mathcal{G}}^{k\text{-(I)GWL}}$. Then, $\mathtt{POOL}$ \textbf{maintains expressivity}.
\end{proposition}

This proposition establishes a simple sufficient condition: if the combined selection and reduction functions are injective (one-to-one), then pooling preserves distinguishability. In other words, if different sets of node features always get mapped to different supernode features, no discriminative information is lost during coarsening.
We expand this to geometric graphs:
\begin{theorem}\label{thm:maintain_expressivity}
     Let $\mathcal{G}_{1} = (\mathcal{A}_{1}, \mathbf{S}_{1}, \vec{\mathbf{V}_{1}}, \vec{\mathcal{X}_{1}})$ and $\mathcal{G}_{2} = (\mathcal{A}_{2}, \mathbf{S}_{2}, \vec{\mathbf{V}_{2}}, \vec{\mathcal{X}_{2}})$ be geometric graphs with $|\mathbf{V}_{1}| = |\mathbf{V}_{2}| = N$ such that $\mathcal{G}_{1} \neq_{k\text{-(I)GWL}} \mathcal{G}_{2}$. Let $\mathtt{POOL} = (\mathtt{SEL}, \mathtt{RED}, \mathtt{CON})$ be a graph pooling layer placed after a block of $L$ message-passing layers such that $\mathcal{G}_{1P} = \mathtt{POOL}(\mathcal{G}_{1}^{L})$ and $\mathcal{G}_{2P}=\mathtt{POOL}(\mathcal{G}_{2}^{L})$ with $|\mathbf{V}_{1P}| = |\mathbf{V}_{2P}| = K$. Let $\mathbf{S}_{1/2}^{L}$, $\vec{\mathbf{V}}_{1/2}^{L}\in \mathbb{R}^{N\times d}$ and $\mathbf{S}_{1/2}^{P}$, $\vec{\mathbf{V}}_{1/2}^{P}\in \mathbb{R}^{K\times f}$ be the node features before and after $\mathtt{POOL}$ respectively.

    Then, $\mathcal{G}_{1P}$ and $\mathcal{G}_{2P}$ will be $k$-(I)GWL-distinguishable if the following conditions hold:

\begin{enumerate}
    \item $\sum_{i}^{N} s_{1_{i}}^{L}\neq \sum_{i}^{N} s_{2_{i}}^{L}$ or $\sum_{i}^{N}\vec{\mathbf{v}}_{1_{i}}^{L}\neq \sum_{i}^{N}\vec{\mathbf{v}}_{2_{i}}^{L}$
    \item The memberships generated by $\texttt{SEL}$ satisfy $\sum_{j}^{K} c_{ij} = \lambda$, with $\lambda > 0$ for each node $i$, i.e., the cluster assignment matrix $\mathbf{C}$ is a right stochastic matrix up to the global constant $\lambda$
    \item The function $\texttt{RED}$ is of type $\texttt{RED}: (\mathbf{S}^{L}, \vec{\mathbf{V}}^{L}, \mathbf{C}) \mapsto (\mathbf{S}^{P} = \mathbf{C}^{\intercal}\mathbf{S}^{L}, \vec{\mathbf{V}}^{P} = \mathbf{C}^{\intercal}\vec{\mathbf{V}}^{L})$
\end{enumerate}
\end{theorem}

Theorem \ref{thm:maintain_expressivity} gives sufficient conditions ensuring pooling preserves the GNN's ability to distinguish graphs: if two graphs are distinguishable pre-pooling, they remain so post-pooling. The requirements are: (1) the pre-pooling message passing yields distinct multisets for distinguishable inputs, (2) every original node is assigned (right-stochastic membership up to constant $\lambda$), and (3) $\texttt{RED}$ applies convex combination $(\mathbf{S}^P,\vec{\mathbf{V}}^P)=\mathbf{C}^T(\mathbf{S}^L,\vec{\mathbf{V}}^L)$. 

Together, these conditions ensure that if two protein graphs have different structural signatures before pooling, those signatures are preserved in the coarsened representation: Condition (1) ensures input graphs are already distinguishable, Condition (2) guarantees every node contributes to some supernode (no information is discarded), and Condition (3) requires supernode features are formed by weighted combinations of their constituents. Notably, the way we connect supernodes (via $\texttt{CON}$) does not affect this guarantee. See Appendix \ref{sec:proof-maintain-expressivity} for the proof.

For increasing expressivity, any $\texttt{SEL}$ that separates previously $k$-(I)GWL-indistinguishable graphs can be paired with an injective $\texttt{RED}$ to form a $\texttt{POOL}$ that strictly improves distinguishing power. We formalize this in the following lemma and theorem:

\begin{lemma}\label{lmma:injective_red}
    Let $\mathtt{RED}$ be injective function. For any two geometric graphs $\mathcal{G}_1$ and $\mathcal{G}_2$ we have 
    $$\mathtt{SEL}(\mathcal{G}_1) \neq_{k\text{-(I)GWL}} \mathtt{SEL}(\mathcal{G}_2) \implies \mathtt{POOL}(\mathcal{G}_1) \neq_{k\text{-(I)GWL}}\mathtt{POOL}(\mathcal{G}_2).$$
\end{lemma}

\begin{theorem}\label{thm:increase_expressivity}
    Let $\mathtt{SEL}$ be a function such that it distinguishes any pair of GWL- or IGWL-indistinguishable graphs, i.e., $\exists \mathcal{G}_{1}, \mathcal{G}_{2}.\ \mathcal{G}_{1} =_{k\text{-(I)GWL}} \mathcal{G}_{2} \Rightarrow \mathtt{SEL}(\mathcal{G}_{1}) \neq_{k\text{-(I)GWL}} \mathtt{SEL}(\mathcal{G}_{2})$

    Then, a geometric pooling operator $\mathtt{POOL}$ exists that \textbf{increases expressivity}.
\end{theorem}

The lemma states that if selection creates distinguishable supernode groupings and reduction is injective, then the full pooling operation preserves that distinguishability. The theorem then shows that clever selection can turn indistinguishable graphs into distinguishable ones. This is powerful: it means that pooling is not merely a computational trick to reduce graph size, but can actually reveal structural differences that were hidden at the atomic scale. Our k-chains experiment (Section \ref{sec: se}) demonstrates this empirically: pooled graphs become distinguishable with fewer message-passing layers than would otherwise be required.
See Appendix \ref{sec:proof-increase-expressivity} for the proofs.

\section{Experimental Results}

\subsection{Fold Classification}

We evaluate on protein fold classification to demonstrate the effectiveness of our proposed Geometric Graph U-Net.
SchNet \cite{schutt2017schnet} and GVP \cite{jing2020learning} serve as baselines, tuned with decoders with fewer hidden layers but higer performance. Then, we introduce U-Net variants with pooling and unpooling layers.
We employ the SCOP 1.75 dataset \cite{hou2018deepsf} for model training on a single NVIDIA RTX 4090 with up to 150 epochs using Adam and learning rates in [0.001, 0.01].
Evaluation follows the Fold Prediction task \cite{jamasb2024evaluating}, classifying proteins into 1,195 folds. Results are reported on three test subsets (Fold, Family, Superfamily) as micro-averaged accuracy and F1 score over three runs.

As shown in Table \ref{tbl:fold-classification-results}, our Geometric U-Net model provides substantial and consistent performance gains across both invariant (SchNet) and equivariant (GVP) backbones. For instance, the GVP U-Net with Point Pooling (U-Net Point Pool) improves the Superfamily F1 score from 0.793 to 0.849 with few additional parameters, a significant gain on this task. 
This demonstrates that the hierarchical architecture, rather than the specifics of the message-passing layer, is the primary driver of the improved ability to capture global structural information.
For a fairer comparison, the U-Net variants remove one decoder layer to match the baseline parameter numbers. Despite the lighter decoders, they still achieve performance gains, demonstrating the efficacy of our approach.

\begin{table}[t]
\centering
\caption{\textbf{Benchmarking Geometric Graph U-Nets on protein fold classification.} Hierarchical pooling with U-Net architectures consistently improves performance across both invariant (SchNet) and equivariant (GVP) backbones without significantly increasing model parameters. The best results are highlighted in \textbf{bold}, with improved performance from the baseline indicated in \colorbox{green!10}{green}.}
\begin{adjustbox}{max width=\textwidth}
\begin{tabular}{@{}llccccccccc@{}} \toprule
\multirow{2}{*}{\textbf{Model}} & \multirow{2}{*}{\textbf{Architecture}} & \multirow{2}{*}{\textbf{Params}} & \multicolumn{3}{c}{\textbf{Accuracy}} & \multicolumn{3}{c}{\textbf{F1}} \\ 
\cmidrule(lr){4-6} \cmidrule(lr){7-9}
& & & \textbf{Fold} & \textbf{Superfamily} & \textbf{Family} & \textbf{Fold} & \textbf{Superfamily} & \textbf{Family} \\
\midrule
\multirow{\iftoggle{advanced}{7}{3}}{*}{\textbf{SchNet}}
& Baseline          & 5.23M & 0.331 \pmn{0.014} & 0.442 \pmn{0.012} & 0.928 \pmn{0.013} & 0.087 \pmn{0.008} & 0.157 \pmn{0.005} & 0.688 \pmn{0.037} \\
\iftoggle{advanced}{
& SchNet-Adv        & 14.00M & 0.311 & 0.369 & 0.889 & 0.366 & 0.444 & 0.919 \\
}
& U-Net Point Pool      & 5.23M & 0.314 \pmn{0.007} & 0.436 \pmn{0.014} & \cellcolor{green!10} 0.937 \pmn{0.007} & 0.081 \pmn{0.001} & 0.156 \pmn{0.014} & \cellcolor{green!10} 0.715 \pmn{0.033} \\
& U-Net Sparse Pool      & 5.23M & \cellcolor{green!10} \textbf{0.340 \pmn{0.010}} & \cellcolor{green!10} \textbf{0.461 \pmn{0.002}} & \cellcolor{green!10} \textbf{0.948 \pmn{0.003}} & \cellcolor{green!10} \textbf{0.095 \pmn{0.007}} & \cellcolor{green!10} \textbf{0.170 \pmn{0.004}} & \cellcolor{green!10} \textbf{0.759 \pmn{0.018}} \\
\iftoggle{advanced}{
& U-Net-Adv-PP      & 7.18M & 0.286 \pmn{0.016} & 0.419 \pmn{0.015} & \cellcolor{green!10} 0.938 \pmn{0.006} & 0.075 \pmn{0.005} & 0.153 \pmn{0.009} & \cellcolor{green!10} 0.714 \pmn{0.018} \\
& U-Net-Adv-SP      & 7.18M & 0.320 \pmn{0.018} & \cellcolor{green!10} \textbf{0.466 \pmn{0.020}} & \cellcolor{green!10} 0.947 \pmn{0.008} & 0.084 \pmn{0.008} & \cellcolor{green!10} \textbf{0.177 \pmn{0.012}} & \cellcolor{green!10} 0.743 \pmn{0.032} \\
}{}
\midrule
\multirow{\iftoggle{advanced}{5}{3}}{*}{\textbf{GVP}}
& Baseline      & 3.60M & 0.336 \pmn{0.024} & 0.512 \pmn{0.021} & 0.961 \pmn{0.005} & 0.097 \pmn{0.010} & 0.215 \pmn{0.010} & 0.793 \pmn{0.013} \\
& U-Net Point Pool  & 3.06M & \cellcolor{green!10} \textbf{0.404 \pmn{0.002}} & \cellcolor{green!10} 0.514 \pmn{0.007}  & \cellcolor{green!10} \textbf{0.971 \pmn{0.003}}& \cellcolor{green!10} \textbf{0.119 \pmn{0.002}} & 0.206 \pmn{0.001} & \cellcolor{green!10} \textbf{0.849 \pmn{0.010}} \\
& U-Net Sparse Pool  & 3.06M & \cellcolor{green!10} 0.373 \pmn{0.015} & \cellcolor{green!10} \textbf{0.529 \pmn{0.009}} & \cellcolor{green!10} 0.966 \pmn{0.004} & \cellcolor{green!10} 0.111 \pmn{0.005} & \cellcolor{green!10} \textbf{0.217 \pmn{0.004}} & \cellcolor{green!10} 0.830 \pmn{0.019} \\
\iftoggle{advanced}{
& U-Net-Adv-PP  & 3.44M & \cellcolor{green!10} 0.375 \pmn{0.017} & \cellcolor{green!10} 0.522 \pmn{0.015} & \cellcolor{green!10} 0.966 \pmn{0.008} & \cellcolor{green!10} 0.109 \pmn{0.006} & \cellcolor{green!10} \textbf{0.221 \pmn{0.016}} & \cellcolor{green!10} 0.830 \pmn{0.031} \\
& U-Net-Adv-SP  & 3.44M & \cellcolor{green!10} 0.365 \pmn{0.012} & \cellcolor{green!10} 0.517 \pmn{0.013} & \cellcolor{green!10} 0.963 \pmn{0.004} & \cellcolor{green!10} 0.105 \pmn{0.007} & 0.214 \pmn{0.012} & \cellcolor{green!10} 0.811 \pmn{0.024} \\
}{} \bottomrule
\end{tabular}
\end{adjustbox}
\label{tbl:fold-classification-results}
\end{table}


\subsection{Synthetic Experiment on Expressivity: k-chains} \label{sec: se}

Table \ref{tbl:kchains} shows how different pooling layers enhance expressivity on synthetic k-chain graphs from \citet{joshi2023expressive}. We report averages over 10 runs for $k=4$ and assess pooling effectiveness on SchNet \cite{schutt2017schnet}, SphereNet \cite{liu2021spherical}, TFN \cite{thomas2018tensor}, GVP-GNN \cite{jing2020learning}, and EGNN \cite{satorras2021n}.
The GWL test requires $(\lfloor \frac{k}{2} \rfloor + 1)$ iterations to differentiate k-chains. 
Our pooling procedure essentially reduces $k$ by merging the central nodes and adding edges between endpoint second-order neighbours. 
As expected, with pooling, we see all models distinguish k-chains with fewer layers than the minimal requirement. This is a simple yet effective demonstration of Theorem \ref{thm:increase_expressivity}, showing that pooling can enhance expressivity by enabling the model to distinguish geometric graphs it previously could not.

\begin{table}[ht]
\centering
\caption{\textbf{$\mathbf{k}$-chain geometric graphs.} 
Geometric pooling enables all GNN architectures to distinguish k-chains with fewer layers than theoretically required (3 layers), empirically validating Theorem \ref{thm:increase_expressivity} that pooling increases expressivity.
$k$-chains are $(\lfloor \frac{k}{2} \rfloor + 1)$-hop distinguishable and $(\lfloor \frac{k}{2} \rfloor + 1)$ GWL iterations are theoretically sufficient to distinguish them.
Anomalous results are marked in \colorbox{red!10}{red} and expected results in \colorbox{green!10}{green}.
}
\adjustbox{max width=0.9\textwidth}{
\begin{tabular}{clccccc} \toprule
& ($k=\mathbf{4}$-chains) & \multicolumn{5}{c}{\textbf{Number of layers}} \\
& \textbf{GNN Layer} & $\lfloor \frac{k}{2} \rfloor$ & \cellcolor{gray!10} $\lfloor \frac{k}{2} \rfloor + 1 = \mathbf{3}$ & $\lfloor \frac{k}{2} \rfloor + 2$ & $\lfloor \frac{k}{2} \rfloor + 3$ & $\lfloor \frac{k}{2} \rfloor + 4$ \\
\midrule
\multirow{10}{*}{\rotatebox[origin=c]{90}{Equivariant}} &
\gray{GWL} & \gray{50\%} & \gray{\textbf{100\%}} & \gray{\textbf{100\%}} & \gray{\textbf{100\%}} & \gray{\textbf{100\%}} \\
& E-GNN & \cellcolor{red!10} 50.0 \pmn{0.0} & \cellcolor{red!10} 50.0 \pmn{0.0} & \cellcolor{red!10} 50.0 \pmn{0.0} & \cellcolor{red!10} 50.0 \pmn{0.0} & \cellcolor{green!10} \textbf{100.0 \pmn{0.0}} \\
& E-GNN - 3 nodes & \cellcolor{green!10} \textbf{100.0 \pmn{0.0}} & \cellcolor{green!10} \textbf{100.0 \pmn{0.0}} & \cellcolor{green!10} \textbf{100.0 \pmn{0.0}} & \cellcolor{green!10} \textbf{100.0 \pmn{0.0}} & \cellcolor{green!10} \textbf{100.0 \pmn{0.0}} \\
& E-GNN - 4 nodes & \cellcolor{green!10} \textbf{100.0 \pmn{0.0}} & \cellcolor{green!10} \textbf{100.0 \pmn{0.0}} & \cellcolor{green!10} \textbf{100.0 \pmn{0.0}} & \cellcolor{green!10} \textbf{100.0 \pmn{0.0}} & \cellcolor{green!10} \textbf{100.0 \pmn{0.0}} \\
& GVP-GNN & \cellcolor{red!10} 50.0 \pmn{0.0} & \cellcolor{green!10} \textbf{90.0 \pmn{20.0}} & \cellcolor{green!10} \textbf{80.0 \pmn{24.5}} & \cellcolor{green!10} \textbf{100.0 \pmn{0.0}} & \cellcolor{green!10} \textbf{90.0 \pmn{20.0}} \\
& GVP-GNN - 3 nodes & \cellcolor{green!10} \textbf{100.0 \pmn{0.0}} & \cellcolor{green!10} \textbf{95.0 \pmn{15.0}} & \cellcolor{green!10} \textbf{100.0 \pmn{0.0}} & \cellcolor{green!10} \textbf{90.0 \pmn{20.0}} & \cellcolor{green!10} \textbf{100.0 \pmn{0.0}} \\
& GVP-GNN - 4 nodes & \cellcolor{green!10} \textbf{100.0 \pmn{0.0}} & \cellcolor{green!10} \textbf{95.0 \pmn{15.0}} & \cellcolor{green!10} \textbf{100.0 \pmn{0.0}} & \cellcolor{green!10} \textbf{100.0 \pmn{0.0}} & \cellcolor{green!10} \textbf{100.0 \pmn{0.0}} \\
& TFN & \cellcolor{red!10} 50.0 \pmn{0.0} & \cellcolor{red!10} 50.0 \pmn{0.0} & \cellcolor{red!10} 50.0 \pmn{0.0} & \cellcolor{red!10} 50.0 \pmn{0.0} & \cellcolor{red!10} 50.0 \pmn{0.0} \\
& TFN - 3 nodes & \cellcolor{green!10} \textbf{100.0 \pmn{0.0}} & \cellcolor{green!10} \textbf{100.0 \pmn{0.0}} & \cellcolor{green!10} \textbf{100.0 \pmn{0.0}} & \cellcolor{green!10} \textbf{100.0 \pmn{0.0}} & \cellcolor{green!10} \textbf{100.0 \pmn{0.0}} \\
& TFN - 4 nodes & \cellcolor{green!10} \textbf{100.0 \pmn{0.0}} & \cellcolor{green!10} \textbf{100.0 \pmn{0.0}} & \cellcolor{green!10} \textbf{100.0 \pmn{0.0}} & \cellcolor{green!10} \textbf{100.0 \pmn{0.0}} & \cellcolor{green!10} \textbf{100.0 \pmn{0.0}} \\
\midrule

\multirow{7}{*}{\rotatebox[origin=c]{90}{Invariant}} &
\gray{IGWL} & \gray{50\%} & \gray{50\%} & \gray{50\%} & \gray{50\%} & \gray{50\%} \\
& SchNet & \cellcolor{red!10} 50.0 {\scriptsize \pmn{0.0}} & \cellcolor{red!10} 50.0 {\scriptsize \pmn{0.0}} & \cellcolor{red!10} 50.0 {\scriptsize \pmn{0.0}} & \cellcolor{red!10} 50.0 {\scriptsize \pmn{0.0}} & \cellcolor{red!10} 50.0 {\scriptsize \pmn{0.0}} \\
& SchNet - 3 nodes & \cellcolor{green!10} \textbf{80.0 \pmn{24.5}} & \cellcolor{green!10} \textbf{95.0 \pmn{15.0}} & \cellcolor{green!10} \textbf{90.0 \pmn{20.0}} & \cellcolor{green!10} \textbf{95.0 \pmn{15.0}} & \cellcolor{green!10} \textbf{95.0 \pmn{15.0}} \\
& SchNet - 4 nodes & \cellcolor{green!10} \textbf{85.0 \pmn{22.9}} & \cellcolor{green!10} \textbf{95.0 \pmn{15.0}} & \cellcolor{green!10} \textbf{95.0 \pmn{15.0}} & \cellcolor{green!10} \textbf{100.0 \pmn{0.0}} & \cellcolor{green!10} \textbf{95.0 \pmn{15.0}} \\
& SphereNet & \cellcolor{red!10} 50.0 \pmn{0.0} & \cellcolor{red!10} 50.0 \pmn{0.0} & \cellcolor{red!10} 50.0 \pmn{0.0} & \cellcolor{red!10} 50.0 \pmn{0.0} & \cellcolor{red!10} 50.0 \pmn{0.0} \\
& SphereNet - 3 nodes & \cellcolor{green!10} \textbf{100.0 \pmn{0.0}} & \cellcolor{green!10} \textbf{100.0 \pmn{0.0}} & \cellcolor{green!10} \textbf{100.0 \pmn{0.0}} & \cellcolor{green!10} \textbf{100.0 \pmn{0.0}} & \cellcolor{green!10} \textbf{100.0 \pmn{0.0}} \\
& SphereNet - 4 nodes & \cellcolor{green!10} \textbf{100.0 \pmn{0.0}} & \cellcolor{green!10} \textbf{100.0 \pmn{0.0}} & \cellcolor{green!10} \textbf{100.0 \pmn{0.0}} & \cellcolor{green!10} \textbf{100.0 \pmn{0.0}} & \cellcolor{green!10} \textbf{100.0 \pmn{0.0}} \\
\bottomrule
\end{tabular}
}
\label{tbl:kchains}
\end{table}

\section{Conclusion}

In this work, we demonstrated that the architectural design of Geometric GNNs should reflect the inherent hierarchical nature of the biological structures they model. Our Geometric Graph U-Net, provably at least as expressive as standard Geometric GNNs and empirically more powerful, serves as a flexible blueprint for a new generation of multi-scale models, extendable with different supernode selection or reduction methods, like the one introduced in \citet{Zhang2025}. By treating proteins not as flat graphs of atoms, but as nested structural assemblies, we show how multi-scale architectures can yield more powerful representations for predicting function from structure.

\textbf{Future Work.} Our experiments aimed to ablate the effectiveness of the Geometric U-Net on protein fold classification. 
Further validation on a broader range of structure-function prediction tasks is required to assess the generality of our approach compared to state-of-the-art protein-specific GNNs \citep{Gligorijevi2021,zhang2023protein} and protein language models \citep{lin2023evolutionary}.
Additionally, we have currently focussed on predictive tasks with static structures.
Future extensions include modelling biomolecular complexes with hierarchical interactions across multiple scales, as well as changes in conformational states mediated by these interactions.
In particular, we believe Geometric U-Nets may be especially effective as encoders for multi-state inverse design pipelines \citep{joshi2024grnade, abrudan2025multistate}.

\textbf{Open-source code.}
\url{https://github.com/VirtualProteins/GNN_UNet}


\bibliographystyle{unsrtnat}
\bibliography{reference}

\appendix

\section{Author Contributions}
\label{app:contrib}

This work originated as two independent team projects at the University of Cambridge MPhil/Part III course on Geometric Deep Learning (L65), 2023-24. We summarise author contributions as follows:

\textbf{Chang Liu} and \textbf{Vivian Li} served as equal contributing first authors, driving the final research and
publication efforts after the L65 course. CL was the overall project lead and provided GPU resources for the final experiments. CL reproduced the initial results from VR and CL’s L65 project, and developed an improved version of the U-Net architecture proposed by VR and CL. CL wrote the final Background, Method and Experiment sections, and VL wrote the final Abstract, Introduction, Theory, and Conclusion, all of which were based on the initial L65 project report by CL and VR.

\textbf{Linus Leong} contributed to the final protein structure experiments, including validating baselines and hyperparameter tuning. LL contributed to writing the final Experiments section.

\textbf{Vladimir Radenkovic} primarily contributed to the L65 project with CL, which provided the structural foundation, early results, and written material for the final manuscript. VR co-developed the geometric pooling, Sparse Pooling and Point Pooling layers, and U-Net framework as well as their expressivity theory. VR wrote the proofs. VR implemented the initial models and experimental code used for both the protein structure experiments and the k-chains task. VR designed the k-chains experiment.

\textbf{Pietro Liò} and \textbf{Chaitanya Joshi} conceived the project, supervised the research, and contributed to writing the final manuscript.


\section{Proofs}

\subsection{Proof of Proposition \ref{prop:maintain_expressivity}}
\label{sec:proof-proposition}

\begin{proof}
    Consider any two geometric graphs $\mathcal{G}_m$ and $\mathcal{G}_n$ that are $k$-(I)GWL-distinguishable, with corresponding node feature multisets $\mathcal{X}_{\mathcal{G}_m}^{k\text{-(I)GWL}} \neq \mathcal{X}_{\mathcal{G}_n}^{k\text{-(I)GWL}}$. Let $\mathcal{G}_m^P$ and $\mathcal{G}_n^P$ denote their respective pooled graphs after applying $\mathtt{POOL}$.
    
    If $\mathtt{RED} \circ \mathtt{SEL}$ is injective, then it maps these two different node feature multisets to different pooled node feature multisets: 
    $$\mathcal{X}_{\mathcal{G}_m^P}^{k\text{-(I)GWL}} = (\mathtt{RED} \circ \mathtt{SEL})(\mathcal{X}_{\mathcal{G}_m}^{k\text{-(I)GWL}}) \neq (\mathtt{RED} \circ \mathtt{SEL})(\mathcal{X}_{\mathcal{G}_n}^{k\text{-(I)GWL}}) = \mathcal{X}_{\mathcal{G}_n^P}^{k\text{-(I)GWL}}.$$
    Thus, expressivity is maintained, as the IGWL/GWL test can distinguish the two pooled graphs $\mathcal{G}_m^P$ and $\mathcal{G}_n^P$ independent of the choice of $\mathtt{CON}$.
\end{proof}

\subsection{Proof of Theorem \ref{thm:maintain_expressivity}}
\label{sec:proof-maintain-expressivity}

\begin{proof}
    Let $\mathcal{G}_1^L$ and $\mathcal{G}_2^L$ denote the two geometric graphs after $L$ message-passing layers, where both graphs have $N$ nodes (as stated in the theorem). Let $\mathbf{C}_1\in \mathbb{R}^{N\times K}$ and $\mathbf{C}_2\in \mathbb{R}^{N\times K}$ be the cluster assignment matrices generated by $\mathtt{SEL}(\mathcal{G}^L_1)$ and $\mathtt{SEL}(\mathcal{G}^L_2)$, respectively, where $K$ is the number of supernodes (clusters) in the pooled graphs.
    
    Assuming Condition 2 holds, the entries of matrices $\mathbf{C}_1$ and $\mathbf{C}_2$ satisfy:
\begin{equation}
    \sum_{j=1}^{K}c_{1_{ij}}=\lambda, \quad \sum_{j=1}^{K}c_{2_{ij}}=\lambda, \quad \forall i=1,\ldots,N
\end{equation}
where $\lambda > 0$ is a constant (i.e., both matrices are right stochastic up to the constant $\lambda$).

    Assuming Condition 3 holds, the $j$-th row of the pooled feature matrix $\mathbf{S}^P_1$ (which represents the scalar features of supernode $j$ in the pooled graph) is given by $\mathbf{s}^P_{1_j}=\sum_{i=1}^N \mathbf{s}_{1_i}^L\cdot c_{1_{ij}}$. Similarly, the $j$-th row of $\mathbf{S}^P_2$ is $\mathbf{s}^P_{2_j}=\sum_{i=1}^N \mathbf{s}_{2_i}^L\cdot c_{2_{ij}}$.
    
    Suppose for the sake of contradiction that $\mathcal{G}_{1P}$ and $\mathcal{G}_{2P}$ are isomorphic. Then there exists a permutation $\pi:\left\{1, \ldots, K\right\} \rightarrow\left\{1, \ldots, K\right\}$ such that $\mathbf{s}^P_{1_j}=\mathbf{s}^P_{2_{\pi(j)}}$ for all $j=1,\ldots,K$. That is:
\begin{equation}
    \sum_{i=1}^N \mathbf{s}_{1_i}^L\cdot c_{1_{ij}}=\sum_{i=1}^N\mathbf{s}_{2_i}^L\cdot c_{2_{i\pi(j)}}, \quad \forall j=1,\ldots,K 
\end{equation}

    Summing over all $j$:
\begin{align}
  \sum_{j=1}^K \sum_{i=1}^N \mathbf{s}_{1_i}^L \cdot c_{1_{ij}}&=\sum_{j=1}^K \sum_{i=1}^N \mathbf{s}_{2_i}^L \cdot c_{2_{i\pi(j)}} \\
  \sum_{i=1}^N \mathbf{s}_{1_i}^L \cdot \sum_{j=1}^K c_{1_{ij}}&=\sum_{i=1}^N \mathbf{s}_{2_i}^L \cdot \sum_{j=1}^K c_{2_{i\pi(j)}} \tag*{(rearranging)} \\
  \sum_{i=1}^N \mathbf{s}_{1_i}^L \cdot \lambda&=\sum_{i=1}^N \mathbf{s}_{2_i}^L \cdot \lambda \tag*{(Condition 2)} \\
  \sum_{i=1}^N \mathbf{s}_{1_i}^L&=\sum_{i=1}^N \mathbf{s}_{2_i}^L
\end{align}

    However, by Condition 1, we have $\sum_{i=1}^{N} s_{1_{i}}^{L}\neq \sum_{i=1}^{N} s_{2_{i}}^{L}$, which contradicts the above equality. Thus, the underlying attributed coarsened graphs $\mathcal{G}_{1P}$ and $\mathcal{G}_{2P}$ must remain non-isomorphic.
    
    According to Proposition 7 in \citet{joshi2023expressive}, the GWL can distinguish any pair of geometric graphs with non-isomorphic underlying attributed graphs under the assumption that graphs are constructed from point clouds using radial cutoffs. Additionally, geometric graphs with distinct underlying attributed graphs are IGWL-distinguishable.

    The same argument applies to vector features $\vec{\mathbf{V}}$ by replacing scalar features $\mathbf{S}$ throughout the proof, establishing that Condition 1 is sufficient for both scalar and vector features. 
\end{proof}

\subsection{Proofs of Lemma \ref{lmma:injective_red} and Theorem \ref{thm:increase_expressivity}}
\label{sec:proof-increase-expressivity}

\begin{proof}[Proof of Lemma \ref{lmma:injective_red}]
Let $\mathcal{G}_1$ and $\mathcal{G}_2$ be two geometric graphs such that $\mathtt{SEL}$ produces different cluster assignments for them, i.e., $\mathtt{SEL}(\mathcal{G}_1) \neq_{k\text{-(I)GWL}} \mathtt{SEL}(\mathcal{G}_2)$. More precisely, the obtained supernode sets (clusters) $\{C_{1_1}, \ldots, C_{1_K}\}$ and $\{C_{2_1}, \ldots, C_{2_K}\}$ are not equal, where $C_{i_j}$ denotes the $j$-th cluster (supernode) in graph $\mathcal{G}_i$.

Since $\mathtt{RED}$ is injective by assumption, it maps different multisets of node features within clusters to different supernode representations. Therefore, the full pooling operation yields $\mathtt{POOL}(\mathcal{G}_1) = (\mathtt{RED} \circ \mathtt{SEL})(\mathcal{G}_1) \neq_{k\text{-(I)GWL}} (\mathtt{RED} \circ \mathtt{SEL})(\mathcal{G}_2) = \mathtt{POOL}(\mathcal{G}_2)$.
\end{proof}

\begin{proof}[Proof of Theorem \ref{thm:increase_expressivity}]
    Let $\mathcal{G}_1$ and $\mathcal{G}_2$ be two geometric graphs that are $k$-(I)GWL-indistinguishable, i.e., $\mathcal{G}_{1} =_{k\text{-(I)GWL}} \mathcal{G}_{2}$. By the assumption of the theorem, $\mathtt{SEL}$ is a function that distinguishes any pair of $k$-(I)GWL-indistinguishable graphs. Therefore:
\begin{align}
    \mathcal{G}_{1} &=_{k\text{-(I)GWL}} \mathcal{G}_{2} \Rightarrow \\
    \mathtt{SEL}(\mathcal{G}_{1}) &\neq_{k\text{-(I)GWL}} \mathtt{SEL}(\mathcal{G}_{2}) \tag*{(by assumption on $\mathtt{SEL}$)}
\end{align}

    By Lemma \ref{lmma:injective_red}, if we choose an injective function $\mathtt{RED}$, then:
\begin{equation}
    \mathtt{SEL}(\mathcal{G}_{1}) \neq_{k\text{-(I)GWL}} \mathtt{SEL}(\mathcal{G}_{2}) \Rightarrow
    \mathtt{POOL}(\mathcal{G}_1) \neq_{k\text{-(I)GWL}} \mathtt{POOL}(\mathcal{G}_2)
\end{equation}
where $\mathtt{POOL} = \mathtt{RED}\circ \mathtt{SEL}$.

    Combining these steps, we have shown that there exists a pooling operator $\mathtt{POOL}$ that maps $k$-(I)GWL-indistinguishable graphs to $k$-(I)GWL-distinguishable graphs, thus \textbf{increasing expressivity}.
\end{proof}

\end{document}

%% file: styles/math_commands.tex
\usepackage{amsmath,amsfonts,bm}
\usepackage{mathtools}

\def\eqref#1{equation~\ref{#1}}

\def\1{\bm{1}}

\def\vs{{\bm{s}}}

\def\vv{{\bm{v}}}

\def\vx{{\bm{x}}}

\DeclareMathAlphabet{\mathsfit}{\encodingdefault}{\sfdefault}{m}{sl}
\SetMathAlphabet{\mathsfit}{bold}{\encodingdefault}{\sfdefault}{bx}{n}

\newcommand{\defeq}{\vcentcolon=}

\newcommand*{\ldblbrace}{\{\mskip-5mu\{}
\newcommand*{\rdblbrace}{\}\mskip-5mu\}}

\def\1{{\mathbf{1}}}

\def\vs{{\bm{s}}}